\pdfoutput=1
\documentclass[10pt, logo, copyright]{nv}

\usepackage{booktabs} 
\usepackage{caption}  
\usepackage{xcolor}   
\usepackage{multirow}
\usepackage{makecell} 
\usepackage{url}
\usepackage{booktabs}
\usepackage{graphicx}
\usepackage{xspace}
\usepackage{natbib}
\usepackage{amsmath,amsfonts,bm}
\usepackage{amssymb}
\usepackage{amsthm}
\usepackage{algorithm}
\usepackage{algpseudocode}
\usepackage{pifont} 
\usepackage{tikz}
\usetikzlibrary{tikzmark,calc}
\usepackage{wrapfig}
\usepackage{subcaption}

\definecolor{cvprblue}{rgb}{0.21,0.49,0.74}
\definecolor{iccvblue}{rgb}{0.21,0.49,0.74}
\definecolor{mitblue}{rgb}{0.88,0.95,0.96}
\definecolor{gold}{rgb}{0.75,0.6,0.12}
\colorlet{shadecolor}{gray!40}
\definecolor{mydarkred}{rgb}{0.8,0.02,0.02}

\newcolumntype{g}{>{\columncolor{mitblue}}c}
\newcolumntype{f}{>{\columncolor{mitblue}}l}
\newcolumntype{h}{>{\columncolor{mitblue}}r}
\newcolumntype{i}{>{\columncolor{gray}}c}

\def\methodfullcap{Lightning On-Policy Distillation\xspace}

\def\methodterm{Lightning OPD\xspace}

\def\speedupsmall{3.6$\times$\xspace}
\def\speedup{4.0$\times$\xspace}

\theoremstyle{definition}  
\newtheorem{theorem}{Theorem}[section]
\newtheorem{proposition}[theorem]{Proposition}
\newtheorem{corollary}[theorem]{Corollary}

\newtheorem{assumption}[theorem]{Assumption}
\newtheorem{remark}[theorem]{Remark}

\usepackage[most]{tcolorbox}

\title{
 Lightning OPD: Efficient Post-Training for Large Reasoning Models with Offline On-Policy Distillation
}

\author{
Yecheng Wu, Song Han, Han Cai \\~\\
NVIDIA \\
\url{https://github.com/jet-ai-projects/Lightning-OPD}
}

\correspondingauthor{Han Cai (\texttt{hcai@nvidia.com}).}

\begin{abstract}
    \textbf{Abstract:} On-policy distillation (OPD) is an effective post-training paradigm for large language models but requires a live teacher server throughout training, resulting in substantial infrastructure overhead. We investigate whether OPD can be performed offline by precomputing teacher log-probabilities once over SFT rollouts and reusing them during training. We find that naively doing so fails to reliably match standard OPD, and trace the root cause to a previously overlooked condition we term \textbf{teacher consistency}, requiring that the same teacher be used for both supervised fine-tuning and OPD. Violating this condition introduces a gradient bias that degrades performance for both offline and online OPD. Building on this insight, we propose \textbf{Lightning OPD}, an offline on-policy distillation framework that enforces teacher consistency and eliminates the need for a live teacher server entirely. We prove that, under teacher consistency, Lightning OPD shares the same optimum as standard OPD, with bounded gradient discrepancy and an implicit regularization effect that helps prevent policy drift. Experiments on math reasoning and code generation show that Lightning OPD achieves comparable performance to standard OPD while delivering \textbf{\speedup} higher training efficiency. Starting from an SFT-initialized Qwen3-8B-Base model, Lightning OPD reaches \textbf{69.9\% on AIME 2024} in just \textbf{30 GPU hours}. Lightning OPD further scales to MoE architectures, training Qwen3-30B-A3B to \textbf{71.0\% on AIME 2024} on a \textbf{single 8$\times$H100 node}, substantially lowering the barrier for academic research on LLM post-training. Our code is released at \url{https://github.com/jet-ai-projects/Lightning-OPD}.
    \end{abstract}

\begin{document}
\maketitle

\begin{figure*}[ht]
\centering
\includegraphics[width=\linewidth]{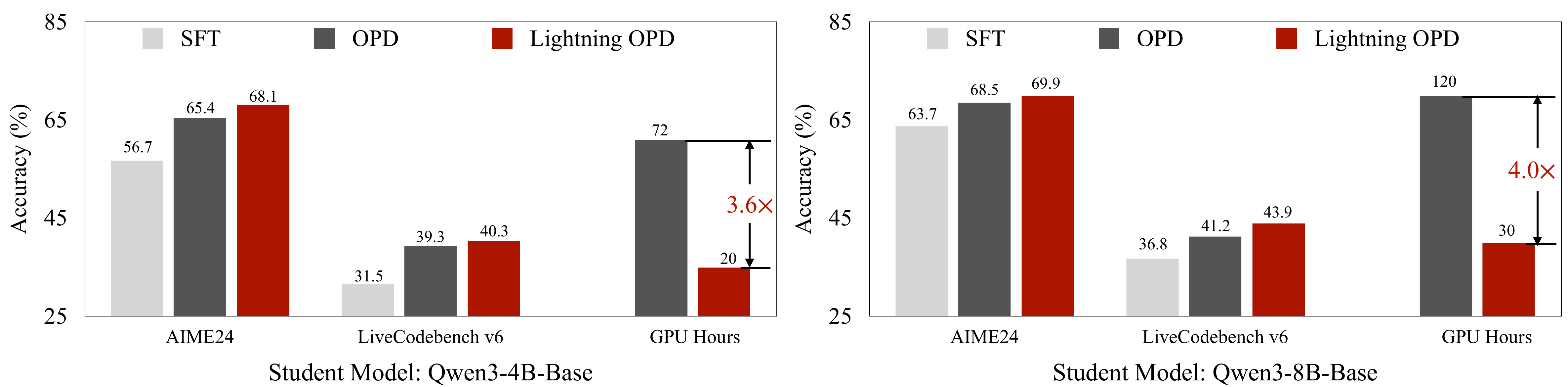}
\vspace{0.2cm}
\includegraphics[width=\linewidth]{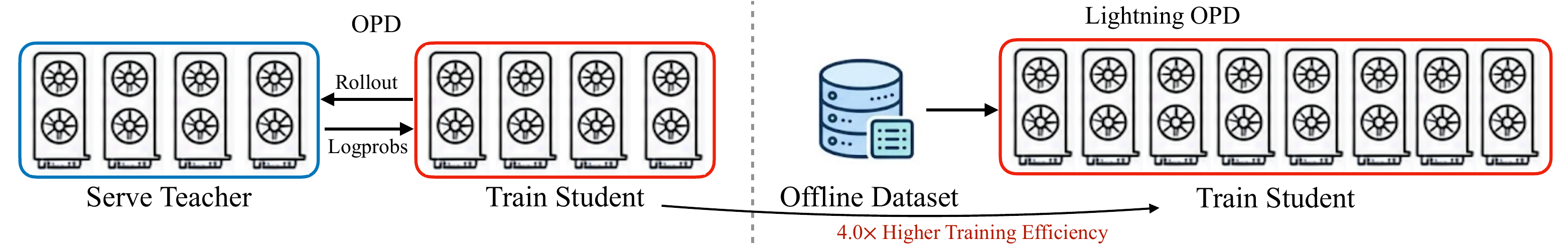}
\caption{(Top) Performance (Pass@1, \%) and training cost of \methodterm compared to standard OPD and the SFT baseline on Qwen3-4B-Base and Qwen3-8B-Base models. \methodterm achieves comparable performance to standard OPD on both math and coding benchmarks across both scales, while eliminating the need for a live teacher server during training. At the 8B scale, \methodterm achieves a state-of-the-art \textbf{69.9\%} on AIME 2024 in just \textbf{30 GPU hours}, delivering \textbf{\speedup{}} higher training efficiency than standard OPD. (Bottom) Intuitive comparison of GPU resource allocation. Standard OPD requires co-hosting the student and teacher, fragmenting GPU resources. \methodterm collects rollouts and teacher log-probabilities offline, dedicating all GPUs to student training.}
\label{fig:teaser}
\end{figure*}

\section{Introduction}
\label{sec:intro}

Large language models (LLMs) have achieved remarkable progress across tasks such as mathematical reasoning, code generation, and multi-step agent planning \cite{guo2025deepseek,singh2025openai,team2026kimi,nvidia_nemotron_3_2025,xiao2026mimo}. This success is underpinned by carefully designed post-training pipelines \cite{ouyang2022training}, which typically consist of supervised fine-tuning (SFT) on high-quality data \cite{guha2025openthoughts}, followed by a reinforcement learning (RL) stage to elicit stronger reasoning capabilities. On-Policy Distillation (OPD) \cite{agarwal2024policy,lu2025onpolicydistillation,yang2026learning,zhao2026self,shenfeld2026self,hubotter2026reinforcement,Nemotron_Cascade_2,song2026survey,chen2026rethinking,fu2026revisiting,wang2026entropy} has emerged as a particularly effective alternative to the RL stage. It trains a student model to match a stronger teacher’s token-level distribution using dense per-token advantage signals. Compared to Reinforcement Learning from Verifiable Rewards (RLVR) \cite{schulman2017proximal,shao2024deepseekmath,guo2025deepseek,yu2025dapo}, OPD provides richer supervision, offers greater training stability, and incurs significantly lower training costs, while achieving competitive or superior performance across a wide range of tasks \cite{lu2025onpolicydistillation,yang2025qwen3,yang2026learning,Nemotron_Cascade_2,zeng2026glm5,xiao2026mimo}.

However, standard OPD requires the teacher to score every student rollout during training, which introduces a persistent infrastructure bottleneck. A dedicated multi-GPU teacher server must run in parallel with the training job, leading to substantial compute overhead and making large-scale experiments costly and difficult to reproduce, particularly for academic researchers without access to extensive serving infrastructure.

A natural question is whether the benefits of on-policy supervision can be preserved while eliminating the need for a live teacher server. On-policy training is defined by the student’s current rollout distribution, which evolves at every gradient step, making the teacher appear indispensable. However, recent empirical studies suggest that RL-trained models remain surprisingly close to their SFT initialization: reasoning trajectories in RL models are largely a reweighted subset of those present in the SFT model~\cite{yue2025does}, and on-policy updates are inherently biased toward solutions that minimize KL divergence from the reference policy~\cite{shenfeld2025razor}. We observe a similar phenomenon in OPD training, where the student’s distribution exhibits only modest drift from the SFT reference throughout the OPD stage. This observation suggests a practical offline alternative \cite{rang2025revealing}: precompute the teacher’s log-probabilities once over SFT rollouts prior to training and reuse these values throughout the OPD process, thereby eliminating the need for a live teacher server.

In practice, however, naively applying this offline precomputation fails to reliably match the performance of standard OPD. Investigating the root cause, we find that the issue does not primarily stem from the offline approximation itself, but rather from a more fundamental condition that has been overlooked in prior OPD work, which we term \textbf{\emph{teacher consistency}}. Unlike RLVR, where model behavior is shaped solely by a reward signal, OPD involves \emph{two} distinct teachers: one used during the SFT stage to generate training trajectories, and another used during the OPD stage to provide the reference distribution. Teacher consistency requires these two teachers to be the same model. In practice, existing pipelines often violate this condition by following conventions inherited from RLVR, where SFT datasets are curated using whichever teacher produces the highest-quality demonstrations, without regard to the teacher used during OPD. For example, Thinking Machines Lab~\cite{lu2025onpolicydistillation} trains a Qwen3-8B-Base model on OpenThoughts-3~\cite{guha2025openthoughts}, whose trajectories are generated by QwQ-32B, while using Qwen3-32B as the OPD teacher, resulting in a mismatch that our analysis predicts to be detrimental. We show that such teacher inconsistency introduces a gradient bias that degrades performance for both offline and online OPD, with a more pronounced effect on the offline variant. These findings establish teacher consistency as an important design principle for any OPD pipeline, rather than specific to the offline setting.

\looseness=-1
With teacher consistency established as the key condition, we propose \textbf{Lightning OPD} (\textbf{L}ightning \textbf{O}n-\textbf{P}olicy \textbf{D}istillation), an offline distillation framework that naturally arises from enforcing this principle. In the SFT stage, the base model is fine-tuned on trajectories generated by a chosen teacher $\pi_T$ to obtain the reference policy $\pi_{\text{ref}}$. In the OPD stage, rollouts are sampled from $\pi_{\text{ref}}$, and the same teacher’s log-probabilities are precomputed once over these fixed responses, eliminating the need for a live teacher server during training. We provide a rigorous theoretical analysis showing that, under teacher consistency, Lightning OPD provably shares the same optimum as standard OPD. Moreover, the gradient discrepancy between the two remains bounded throughout training, and the offline objective introduces an implicit regularization effect that naturally prevents policy drift without requiring any explicit penalty.

We evaluate Lightning OPD on math and code reasoning tasks across diverse student–teacher model pairs, including Qwen3-4B with Qwen3-8B and Qwen3-8B with Qwen3-32B as the teacher \cite{yang2025qwen3}. Lightning OPD achieves comparable performance to standard OPD across all benchmarks while bringing significant training speedup by removing the need for parallel teacher serving infrastructure. Concretely, starting from an SFT-initialized Qwen3-8B-Base model, Lightning OPD reaches \textbf{69.9\% on AIME 2024} in just \textbf{30 GPU hours}, achieving \textbf{\speedup} higher training efficiency than standard OPD. We further demonstrate that Lightning OPD scales to Mixture-of-Experts (MoE) architectures, where standard OPD faces significant challenge due to the memory overhead of co-hosting both student and teacher models. On a single 8$\times$H100 node, Lightning OPD trains a Qwen3-30B-A3B model to \textbf{71.0\% on AIME 2024} and \textbf{60.8\% on LiveCodeBench v5}, significantly lowering the barrier for academic research on LLM post-training.

Our main contributions are summarized as follows:
\begin{itemize}
\item We identify \emph{teacher consistency} as an important design principle for effective OPD, requiring that SFT teacher and OPD teacher be the same model. We show that violating this condition introduces a gradient bias that degrades performance for both online and offline OPD.

\item We propose Lightning OPD, an offline on-policy distillation framework that enforces teacher consistency and eliminates the need for a live teacher server by precomputing teacher log-probabilities once over SFT rollouts. We prove that, under teacher consistency, Lightning OPD shares the same optimum as standard OPD, with bounded gradient discrepancy and an implicit regularization effect that naturally prevents policy drift.

\item We validate Lightning OPD on math and code reasoning tasks across diverse student–teacher pairs, from dense models (4B, 8B) to Mixture-of-Experts (30B-A3B). Starting from an SFT-initialized Qwen3-8B-Base model, Lightning OPD achieves \textbf{69.9\% on AIME 2024} in just \textbf{30 GPU hours}, delivering \textbf{\speedup} higher training efficiency than standard OPD. On a 30B MoE model, Lightning OPD reaches \textbf{71.0\% on AIME 2024} on a \textbf{single 8$\times$H100 node}, significantly lowering the barrier for academic research on LLM post-training.
\end{itemize}

\section{Related Work}
\label{sec:related}

\paragraph{LLM Post-Training.}
Post-training is now central to capable LLMs, typically combining supervised fine-tuning (SFT) on high-quality demonstrations \cite{ouyang2022training,guha2025openthoughts,wei2022chain} with reinforcement learning (RL) for further capability elicitation. Existing RL methods span outcome-based optimization against sparse verifiable rewards \cite{hu2025reinforcepp,ahmadian2024back,li2023remax,liu2025drgrpo,zheng2025gspo,minimax2025cispo,shao2024deepseekmath,yu2025dapo} and process-based optimization with dense intermediate supervision \cite{cui2025prime,yuan2025vcppo,yue2025vapo,kazemnejad2024vineppo,schulman2017proximal}, together driving recent reasoning advances \cite{guo2025deepseek,liu2024deepseekv3,team2025kimik15,singh2025openai,team2026kimi,nvidia_nemotron_3_2025,xiao2026mimo,hu2025openreasoner,yang2025qwen3,yang2024qwen25math}. Recent work increasingly studies the co-design of SFT and RL \cite{yan2025luffy,ma2026relift,zhang2025chord,chen2025bridge,liu2025uft,huang2025blend}. From the OPD perspective, our contribution is to show that SFT and OPD must share the same teacher, making teacher consistency a principled co-design constraint.

\paragraph{On-Policy Distillation.}
Knowledge distillation transfers capabilities from a large teacher to a smaller student by matching output distributions \cite{hinton2015distilling}. Standard offline KD uses fixed teacher-generated data \cite{kim2016sequence,gu2024minillm}, whereas on-policy distillation (OPD) aligns the student to the teacher on student-generated rollouts \cite{agarwal2024policy,lu2025onpolicydistillation}, often yielding stronger post-training gains \cite{yang2025qwen3,yang2026learning}. Recent OPD variants study reward extrapolation \cite{yang2026learning}, self-distillation \cite{zhao2026self,hubotter2026reinforcement,shenfeld2026self}, entropy-aware divergence \cite{wang2026entropy}, reinforcement-aware distillation \cite{xu2026reinforcement}, failure modes and recipes \cite{fu2026revisiting,chen2026rethinking}, controllable reasoning \cite{liang2026orbit}, black-box or privileged settings \cite{ye2025blackbox,penaloza2026privileged}, and broader post-training validation \cite{Nemotron_Cascade_2}. Our Lightning OPD removes the live teacher server while preserving dense per-token supervision, and highlights a design constraint largely ignored in prior OPD practice that the SFT-stage and OPD-stage teachers must be the same model.

\looseness=-1
\paragraph{Off-Policy Reinforcement Learning.}
Off-policy RL decouples data collection from optimization \cite{levine2020offline,williams1992reinforce}. This motivates offline RL algorithms \cite{kumar2020cql,kostrikov2021iql,fujimoto2019bcq}, their use in LLM alignment \cite{rafailov2023direct,yuan2023scaling,dong2023raft}, and recent off-policy reasoning methods \cite{schulman2017proximal,shao2024deepseekmath,ritter2026llms,lu2026pclreasoner}. Lightning OPD is similar only in optimizing over a fixed dataset but receives dense token-level teacher supervision rather than sparse rewards, so its central challenge is teacher consistency rather than OOD value estimation. We provide a more detailed comparison in Appendix~\ref{app:discussion}.

\section{Methodology}
\label{sec:method}

\subsection{Preliminaries}

Let $\pi_T$ denote a fixed teacher model and $\pi_\theta$ a trainable student. Given a prompt $q$, a response $x = (a_1, \ldots, a_T)$ is generated autoregressively as
\begin{equation}
    \pi_\theta(x \mid q) = \prod_{t=1}^{T} \pi_\theta(a_t \mid s_t),
\end{equation}
where $s_t = (q, a_1, \ldots, a_{t-1})$. Let $\pi_{\text{ref}}$ denote the SFT-initialized student. The per-token OPD advantage is
\begin{equation}
    A_t(\theta) = \log \pi_T(a_t \mid s_t) - \log \pi_\theta(a_t \mid s_t),
\end{equation}
which is positive where the teacher is more confident than the student and negative otherwise. Standard OPD~\cite{agarwal2024policy,lu2025onpolicydistillation,yang2025qwen3} optimizes
\begin{equation}
    J_{\text{on}}(\theta) = \mathbb{E}_{q \sim p,\; x \sim \pi_\theta}\!\left[\sum_{t=1}^{T} A_t(\theta)\right],
\end{equation}
while Lightning OPD fixes the rollout distribution to $\pi_{\text{ref}}$ and optimizes
\begin{equation}
    J_{\text{off}}(\theta) = \mathbb{E}_{q \sim p,\; x \sim \pi_{\text{ref}}}\!\left[\sum_{t=1}^{T} A_t(\theta)\right].
\end{equation}
Both objectives share the same advantage function and differ only in the response distribution. Following standard OPD practice~\cite{agarwal2024policy,lu2025onpolicydistillation,yang2025qwen3}, the advantage $A_t(\theta)$ is treated as a fixed scalar (stop-gradient) when computing parameter updates; throughout our analysis, $\nabla J_{\text{on}}$ and $\nabla J_{\text{off}}$ denote the resulting advantage-weighted policy gradients (formalized in Appendix~\ref{app:is_decomp}). Defining the per-trajectory gradient $f(x;\theta) := \sum_t A_t(\theta) \cdot \nabla \log \pi_\theta(a_t \mid s_t)$, an IS decomposition gives $\nabla J_{\text{on}}(\theta) = \mathbb{E}_{x \sim \pi_{\text{ref}}}[w(x;\theta) \cdot f(x;\theta)]$ where $w(x;\theta) = \pi_\theta(x)/\pi_{\text{ref}}(x)$, and the offline gradient $\nabla J_{\text{off}}(\theta) = \mathbb{E}_{x \sim \pi_{\text{ref}}}[f(x;\theta)]$ is the special case $w \equiv 1$.

\subsection{\methodfullcap}
\label{subsec:pipeline}

\begin{figure*}[t]
\centering
\includegraphics[width=\linewidth]{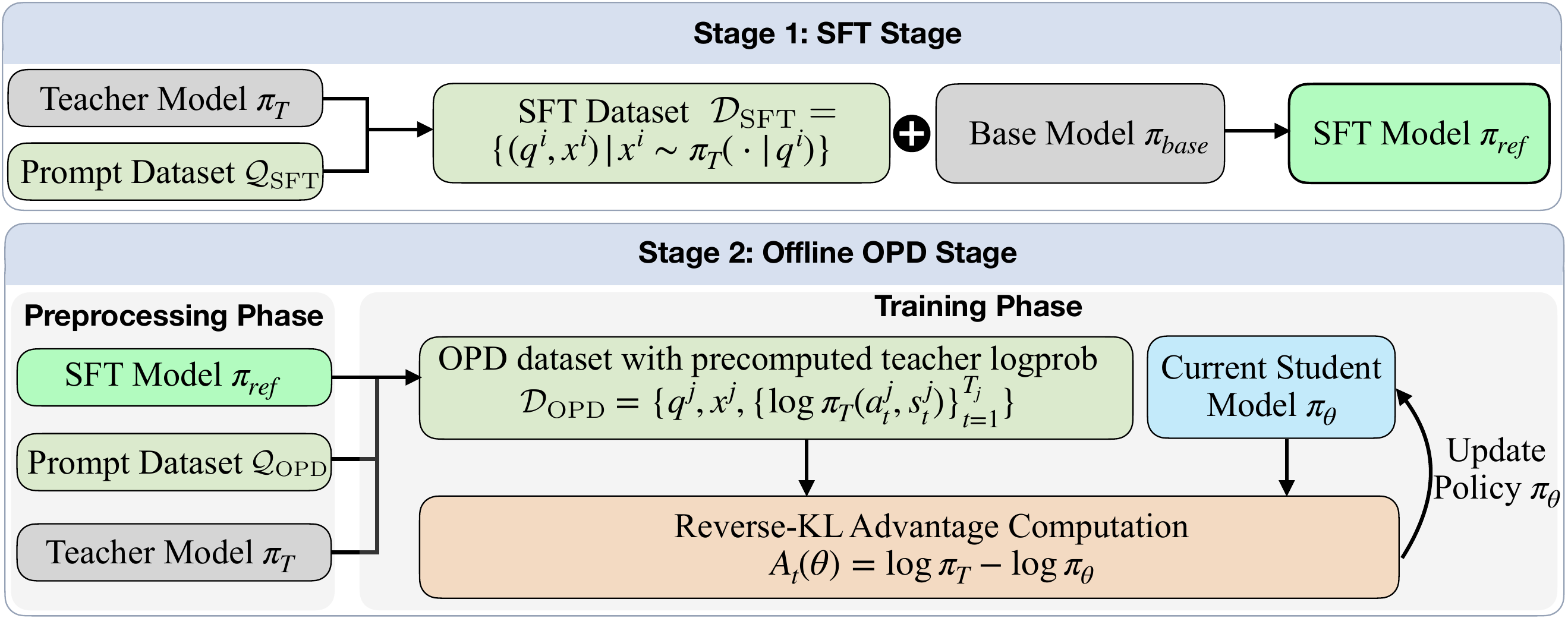}
\caption{
\looseness=-1
Overview of \methodterm. In the SFT stage, the teacher $\pi_T$ generates trajectories on $\mathcal{Q}_{\text{SFT}}$ then the base model $\pi_{\text{base}}$ is fine-tuned on these trajectories to obtain the SFT model $\pi_{\text{ref}}$. OPD stage proceeds in two phases. In the preprocessing phase, rollouts are sampled from $\pi_{\text{ref}}$ on $\mathcal{Q}_{\text{OPD}}$ and the teacher is queried once to compute and store per-token log-probabilities $\log \pi_T(a_t \mid s_t)$, forming the offline dataset $\mathcal{D}_{\text{OPD}}$. In the training phase, the student $\pi_\theta$ is initialized from $\pi_{\text{ref}}$ and trained on $\mathcal{D}_{\text{OPD}}$. At each step, the per-token advantage $A_t = \log \pi_T(a_t \mid s_t) - \log \pi_\theta(a_t \mid s_t)$ is computed by reading $\log \pi_T$ from $\mathcal{D}_{\text{OPD}}$ while computing $\log \pi_\theta$ online. During this process, we do not need a live teacher model. In contrast, standard OPD requires a live teacher server throughout training to compute advantages at every gradient step. By eliminating this requirement, \methodterm reduces total training cost of OPD by \textbf{\speedup} at the 8B scale (from 120 to 30 GPU hours).}
\label{fig:overview}
\end{figure*}

Following the common practice of LLM post-training~\cite{ouyang2022training,guo2025deepseek}, \methodterm consists of two stages. We describe each stage below and highlight where \methodterm departs from standard OPD.

\paragraph{Stage 1: Supervised Fine-Tuning.}
Given a base model $\pi_{\text{base}}$, a teacher $\pi_T$, and a prompt dataset $\mathcal{Q}_{\text{SFT}}$, we collect a supervised dataset from teacher-generated trajectories:
\begin{equation}
    \mathcal{D}_{\text{SFT}} = \bigl\{\,\bigl(q^i,\; x^i\bigr) \;\bigm|\; x^i \sim \pi_T(\cdot \mid q^i),\; q^i \in \mathcal{Q}_{\text{SFT}}\,\bigr\}.
\end{equation}
The base model is then fine-tuned on $\mathcal{D}_{\text{SFT}}$ via maximum likelihood estimation to obtain the reference policy $\pi_{\text{ref}}$:
\begin{equation}
    \pi_{\text{ref}} = \operatorname*{arg\,max}_\theta \;\mathbb{E}_{(q,x)\sim\mathcal{D}_{\text{SFT}}}\!\left[\sum_{t=1}^{T} \log \pi_\theta(a_t \mid s_t)\right].
\end{equation}
While standard OPD may use any high-quality SFT dataset regardless of its source, \methodterm requires that $\mathcal{D}_{\text{SFT}}$ be generated by the same teacher $\pi_T$ used in the OPD stage. This \emph{teacher consistency} condition is a prerequisite for the offline approximation to be sound, as we formally establish in Section~\ref{sec:method}.

\paragraph{Stage 2: Offline On-Policy Distillation.}
Starting from a task-specific prompt dataset $\mathcal{Q}_{\text{OPD}}$, the OPD stage proceeds in two phases. In the \textbf{preprocessing phase}, rollouts are sampled from $\pi_{\text{ref}}$ and the teacher is queried \emph{once} to precompute and store per-token log-probabilities, forming the offline dataset:
\begin{equation}
    \mathcal{D}_{\text{OPD}} = \Bigl\{\,\bigl(q^j,\; x^j,\; \{\log \pi_T(a_t^j \mid s_t^j)\}_{t=1}^{T_j}\bigr) \;\Bigm|\; x^j \sim \pi_{\text{ref}}(\cdot \mid q^j),\; q^j\in\mathcal{Q}_{\text{OPD}}\,\Bigr\}.
\end{equation}
In contrast, standard OPD requires a live teacher server throughout training, querying it at every step for log-probabilities on freshly sampled rollouts from the current policy $\pi_\theta$. In the \textbf{training phase}, the student is initialized from $\pi_{\text{ref}}$ and trained on $\mathcal{D}_{\text{OPD}}$ with no teacher server required. At each step, the advantage is computed as $A_t(\theta) = \log \pi_T(a_t \mid s_t) - \log \pi_\theta(a_t \mid s_t)$, where the teacher term is read directly from $\mathcal{D}_{\text{OPD}}$ and the student term is computed online. This eliminates the primary infrastructure bottleneck of standard OPD and makes high-quality on-policy distillation accessible with minimal compute overhead. The full procedure is given in Algorithm~\ref{alg:offline_opd}.

\begin{algorithm}[t]
\caption{Lightning On-Policy Distillation (Lightning OPD)}
\label{alg:offline_opd}
\begin{algorithmic}[1]
\Require Base model $\pi_{\text{base}}$, teacher $\pi_T$, prompt datasets $\mathcal{Q}_{\text{SFT}}$, $\mathcal{Q}_{\text{OPD}}$, learning rate $\eta$
\State \textbf{// Stage 1: SFT}
\State Collect $\mathcal{D}_{\text{SFT}} = \{(q, x) \mid q \in \mathcal{Q}_{\text{SFT}},\; x \sim \pi_T(\cdot \mid q)\}$
\State Fine-tune $\pi_{\text{base}}$ on $\mathcal{D}_{\text{SFT}}$ to obtain $\pi_{\text{ref}}$
\State \textbf{// Stage 2, Phase 1: Preprocessing}
\For{each prompt $q^j \in \mathcal{Q}_{\text{OPD}}$}
    \State Sample $x^j \sim \pi_{\text{ref}}(\cdot \mid q^j)$; store $\log \pi_T(a_t^j \mid s_t^j)$ for all $t$
\EndFor
\State Form $\mathcal{D}_{\text{OPD}} = \{(q^j, x^j, \{\log \pi_T(a_t^j \mid s_t^j)\}_{t=1}^{T_j})\}$
\State \textbf{// Stage 2, Phase 2: Training}
\State Initialize $\pi_\theta \leftarrow \pi_{\text{ref}}$
\For{each training step}
    \State Sample mini-batch from $\mathcal{D}_{\text{OPD}}$
    \State Compute $A_t(\theta) = \log \pi_T(a_t \mid s_t) - \log \pi_\theta(a_t \mid s_t)$, clip to $[-\tau, \tau]$
    \State Update $\theta \leftarrow \theta + \eta \cdot \nabla J_{\text{off}}(\theta)$
\EndFor
\end{algorithmic}
\end{algorithm}

\subsection{Theoretical Analysis}
\label{subsec:theory}

All proofs are deferred to Appendix~\ref{app:proofs}. The analysis rests on three standard assumptions, with a fourth for teacher-mismatch analysis.

\begin{assumption}[Bounded Absolute Advantage]
\label{asm:second_moment}
There exists $\sigma_A < \infty$ such that for all $\theta \in \Theta$: \\
$\mathbb{E}_{x \sim \pi_{\text{ref}}}\!\left[\bigl(\sum_{t=1}^{T} |A_t(\theta)|\bigr)^{2}\right] \leq \sigma_A^2$,
where $A_t(\theta) = \log \pi_T(a_t \mid s_t) - \log \pi_\theta(a_t \mid s_t)$.
\end{assumption}

\begin{assumption}[Support Coverage]
\label{asm:support}
For all $\theta$ encountered during optimization and all prompts $q$: $\operatorname{supp}(\pi_\theta(\cdot \mid q)) \subseteq \operatorname{supp}(\pi_{\text{ref}}(\cdot \mid q))$.
\end{assumption}

\begin{assumption}[Bounded Score Function]
\label{asm:score}
There exists $G < \infty$ such that $\|\nabla \log \pi_\theta(a_t \mid s_t)\|_2 \leq G$ for all $\theta \in \Theta$ and all $t$.
\end{assumption}

\begin{assumption}[Bounded Teacher Mismatch]
\label{asm:mismatch}
Let $\Delta_t = \log \pi_T^{\text{SFT}}(a_t \mid s_t) - \log \pi_T^{\text{OPD}}(a_t \mid s_t)$. There exists $\sigma_\Delta < \infty$ such that
$\mathbb{E}_{x \sim \pi_{\text{ref}}}\!\left[\bigl(\sum_{t=1}^{T} |\Delta_t|\bigr)^{2}\right] \leq \sigma_\Delta^2$.
When teacher consistency holds ($\pi_T^{\text{SFT}} = \pi_T^{\text{OPD}}$), $\Delta_t = 0$ everywhere and $\sigma_\Delta = 0$.
\end{assumption}

Assumptions~\ref{asm:second_moment}--\ref{asm:score} are standard: \ref{asm:second_moment} is an $L^2$ condition on absolute advantages, automatically satisfied under advantage clipping ($|A_t| \leq \tau$ gives $\sigma_A \leq T\tau$); \ref{asm:support} holds naturally when $\pi_\theta$ is initialized from $\pi_{\text{ref}}$; and \ref{asm:score} is standard in policy gradient analyses. Assumption~\ref{asm:mismatch} quantifies the teacher mismatch and is only needed for the teacher consistency results (Theorems~\ref{thm:consistency_gap}--\ref{thm:consistency_online}).

\begin{theorem}[Gradient Discrepancy Bound]
\label{thm:gap}
Under Assumptions~\ref{asm:second_moment}--\ref{asm:score},
\begin{equation}
    \|\nabla J_{\text{on}}(\theta) - \nabla J_{\text{off}}(\theta)\|_2 \leq G \cdot \sigma_A \cdot \sqrt{\chi^2(\pi_\theta \,\|\, \pi_{\text{ref}})},
\end{equation}
where $\chi^2(\pi_\theta \| \pi_{\text{ref}}) = \mathbb{E}_{x \sim \pi_{\text{ref}}}[w(x;\theta)^2] - 1$.
\end{theorem}

At initialization $\pi_\theta = \pi_{\text{ref}}$, $\chi^2 = 0$ and the two gradients coincide exactly; the bound grows with drift but remains small under standard KL regularization. When the teacher is representable, the two methods share a common fixed point:

\begin{theorem}[Shared Fixed Point]
\label{thm:fixed_point}
The online objective satisfies $J_{\text{on}}(\theta) = -\mathrm{KL}(\pi_\theta \,\|\, \pi_T) \leq 0$, with global maximum at $\theta^* \in \arg\min_{\theta \in \Theta} \mathrm{KL}(\pi_\theta \,\|\, \pi_T)$. When $\pi_T \in \Pi_\Theta$, $A_t(\theta^*) = 0$ almost surely, and $\theta^*$ is a shared zero of both the online and offline OPD updates.
\end{theorem}

When $\pi_T \notin \Pi_\Theta$, Theorem~\ref{thm:gap} ensures the offline update stays close to the online one as long as drift remains small, suggesting comparable capacity-limited behavior in practice. Moreover, the two updates are related by a covariance correction:

\begin{theorem}[Gradient Decomposition]
\label{thm:implicit_reg}
The offline gradient decomposes as
$\nabla J_{\text{off}}(\theta) = \nabla J_{\text{on}}(\theta) - \mathrm{Cov}_{\pi_{\text{ref}}}\!\left[w(x;\theta),\; f(x;\theta)\right]$,
where $f(x;\theta) = \sum_t A_t(\theta) \cdot \nabla \log \pi_\theta(a_t \mid s_t)$.
\end{theorem}

The covariance term vanishes at initialization ($w \equiv 1$) and grows with drift, empirically acting as a trust-region effect that stabilizes training without an explicit KL penalty (Figure~\ref{fig:dyn_iw}). The three results above assume the SFT and OPD stages use the same teacher. We formalize this as \emph{teacher consistency} and show that violating it introduces a bias in both paradigms:

\begin{theorem}[Teacher Consistency and the Offline-Online Gap]
\label{thm:consistency_gap}
Let $\pi_T^{\text{SFT}}$ and $\pi_T^{\text{OPD}}$ denote the teachers used in the SFT and OPD stages. Under Assumptions~\ref{asm:second_moment}--\ref{asm:mismatch},
\begin{equation}
    \|\nabla J_{\text{on}}(\theta) - \nabla J_{\text{off}}(\theta)\|_2 \leq G \cdot \!\left(\sigma_A + \sigma_\Delta\right) \cdot \sqrt{\chi^2(\pi_\theta \,\|\, \pi_{\text{ref}})}.
\end{equation}
Additionally, when $\sigma_\Delta > 0$, the mismatched offline gradient can carry a persistent bias bounded by $\|\nabla J_{\text{off}}(\theta) - \nabla J_{\text{off}}^{\delta=0}(\theta)\|_2 \leq G\,\sigma_\Delta$ relative to the consistent offline gradient, independent of $\chi^2$.
\end{theorem}

\begin{theorem}[Teacher Consistency and Standard OPD]
\label{thm:consistency_online}
Let $\nabla J_{\text{on}}^{\delta=0}(\theta)$ denote the standard OPD gradient under a consistent teacher ($\pi_T^{\text{SFT}} = \pi_T^{\text{OPD}}$). At initialization $\theta = \theta_{\text{ref}}$,
\begin{equation}
    \|\nabla J_{\text{on}}(\theta_{\text{ref}}) - \nabla J_{\text{on}}^{\delta=0}(\theta_{\text{ref}})\|_2 \leq G \cdot \sigma_\Delta.
\end{equation}
\end{theorem}

\begin{remark}
Theorems~\ref{thm:consistency_gap}--\ref{thm:consistency_online} show that $G\sigma_\Delta$ biases the gradient of \emph{both} paradigms. Teacher consistency is therefore an important design requirement: when $\sigma_\Delta = 0$, Theorems~\ref{thm:gap}--\ref{thm:fixed_point} apply in full and Lightning OPD matches standard OPD.
\end{remark}
\section{Experiments}
\label{sec:exp}

\subsection{Experimental Setup}
\label{subsec:setup}

\paragraph{Models.}
We train two student models under the \methodterm pipeline, covering different model scales from the Qwen3 model family \cite{yang2025qwen3}. The first uses Qwen3-4B-Base as the student with Qwen3-8B as the teacher. The second uses Qwen3-8B-Base as the student with Qwen3-32B as the teacher. Both pipelines follow the two-stage procedure described in Section~\ref{subsec:pipeline}. The base model is first fine-tuned on teacher-generated trajectories to obtain $\pi_{\text{ref}}$, which is then used to sample rollouts and precompute teacher log-probabilities for the OPD stage.

\paragraph{Training Data.}
The SFT stage uses prompts from OpenThoughts-3~\cite{guha2025openthoughts}, with responses generated by the respective teacher model. For the OPD stage, we train on two domains. Mathematical reasoning uses DAPO-Math-17k~\cite{yu2025dapo}, which provides 17K competition-level math problems spanning a wide range of difficulty. Code generation uses a sampled 30K subset of EpiCoder-func-380k~\cite{wang2025epicoder}, which provides diverse function-level code synthesis problems. For each prompt, we sample a single response from $\pi_{\text{ref}}$ and precompute the corresponding teacher log-probabilities once prior to training, with no teacher server required during the OPD stage.

\paragraph{Benchmarks.}
For math reasoning, we evaluate on AIME 2024~\cite{aimo2024aime}, AIME 2025~\cite{opencompass2025aime}, and HMMT 2025~\cite{balunovic2025matharena}. For code reasoning, we evaluate on LiveCodeBench v5 and v6~\cite{jain2024livecodebench}. In all evaluations, we set the temperature to 0.6, top-$p$ to 0.95, and the maximum generation length to 32,768 for math benchmarks and 40,960 for coding benchmarks. We sample 32 solutions per problem for math benchmarks and 4 solutions per problem for code benchmarks, reporting the average pass@1.

\paragraph{Training Settings.}
The SFT stage is implemented using LlamaFactory~\cite{zheng2024llamafactory} and the OPD stage is implemented using slime~\cite{slime2025}. The OPD stage is trained for 150 steps, which we find sufficient for convergence as shown in Figure~\ref{fig:dyn_aime}. Standard OPD and \methodterm share identical training settings in the OPD stage, differing only in the source of rollouts that standard OPD samples rollouts online from the current student while \methodterm reuses rollouts precomputed from $\pi_{\text{ref}}$ before training begins. Full hyperparameter details are provided in Appendix~\ref{app:hyperparams}.

\subsection{Main Results}

Table~\ref{tab:main} presents evaluation results across five benchmarks for both 4B and 8B model scales. The central finding is that \methodterm, despite eliminating the live teacher server entirely during training, achieves performance on par with standard OPD across all settings, and in several cases marginally exceeds it. This validates our theoretical analysis that under teacher consistency, the offline approximation preserves the same performance optimum as standard OPD. The gains from the OPD stage over the SFT baseline are substantial and consistent across both math and code benchmarks, confirming that on-policy distillation provides strong and transferable post-training improvements. At the 4B scale, compared to ExOPD~\cite{yang2026learning}, a recent OPD baseline, \methodterm achieves substantially better results, reaching 68.1\% compared to 61.0\% on AIME 2024, and the gap widens on code generation where \methodterm reaches 40.3\% on LCB v6 against ExOPD's 29.0\%. At the 8B scale, \methodterm achieves 69.9\% on AIME 2024 and 49.5\% on LiveCodeBench v5. Together, these results demonstrate the effectiveness of \methodterm as a general and efficient post-training framework across model scales and task domains.

\begin{table*}[t]
    \centering
    \caption{Pass@1 on math and code reasoning benchmarks. \methodterm achieves comparable performance to standard OPD across all benchmarks at both model scales, while requiring no live teacher server during training. At the 4B scale, \methodterm substantially outperforms ExOPD~\cite{yang2026learning}, achieving 68.1\% compared to 61.0\% on AIME 2024 and 40.3\% compared to 29.0\% on LCB v6. At the 8B scale, \methodterm reaches 69.9\% on AIME 2024 and 49.5\% on LCB v5. \textbf{Bold} indicates the best result within each model scale.}
    \label{tab:main}
    \setlength{\tabcolsep}{3pt}
    \begin{tabular}{lcccccccc}
    \toprule
    \multirow{2}{*}{\textbf{Method}}
    & \multicolumn{4}{c}{\textbf{Math Reasoning}} & \multicolumn{3}{c}{\textbf{Code Generation}} \\
    \cmidrule(lr){2-5}\cmidrule(lr){6-8}
    & \textbf{AIME 2024} & \textbf{AIME 2025} & \textbf{HMMT 2025} & \textbf{Avg.}
    & \textbf{LCB v5} & \textbf{LCB v6} & \textbf{Avg.} \\
    \midrule
    \multicolumn{8}{l}{\emph{Student: Qwen3-4B-Base\quad Teacher: Qwen3-8B}} \\
    \midrule
    SFT                           & 56.7 & 52.1 & 34.0 & 47.6 & 33.8 & 31.5 & 32.6 \\
    ExOPD~\cite{yang2026learning} & 61.0 & 56.0 & 34.4 & 50.5 & -- & 29.0 & -- \\
    OPD                     & 65.4 & 57.9 & \textbf{39.9} & 54.4 & \textbf{44.2} & 39.3 & \textbf{41.8} \\
    \rowcolor{mitblue}
    \methodterm             & \textbf{68.1} & \textbf{58.4} & 39.8 & \textbf{55.4} & 42.8 & \textbf{40.3} & 41.5 \\
    \midrule
    \multicolumn{8}{l}{\emph{Student: Qwen3-8B-Base\quad Teacher: Qwen3-32B}} \\
    \midrule
    SFT                           & 63.7 & 51.7 & 36.9 & 50.8 & 44.7 & 36.8 & 40.8 \\
    OPD                     & 68.5 & 59.0 & 39.4 & 55.6 & 47.3 & 41.2 & 44.2 \\
    \rowcolor{mitblue}
    \methodterm             & \textbf{69.9} & \textbf{59.2} & \textbf{41.9} & \textbf{57.0} & \textbf{49.5} & \textbf{43.9} & \textbf{46.7} \\
    \bottomrule
    \end{tabular}
    \end{table*}

\subsection{Training Cost}

Table~\ref{tab:cost} compares the training cost of standard OPD and \methodterm. \methodterm achieves a \speedupsmall speedup at the 4B scale, reducing total GPU hours from 72 to just 20, and a \speedup speedup at the 8B scale, bringing the full pipeline down from 120 to just 30 GPU hours. We also provide a per-phase breakdown of the \methodterm pipeline. The actual OPD training stage consumes only a small fraction of this budget, with the remaining cost split between rollout collection and teacher logprob precomputation, both of which are one-time offline operations that require no specialized infrastructure. This stands in sharp contrast to standard OPD, which demands a dedicated multi-GPU teacher server running continuously throughout training. The minimal infrastructure requirement of \methodterm makes high-quality on-policy distillation accessible to practitioners without large-scale training and serving systems.

\begin{table}[t]
    \centering
    \caption{Training cost (GPU hours) of standard OPD vs.\ \methodterm. \methodterm achieves a \speedupsmall speedup at 4B and a \speedup speedup at 8B, requiring only 20 and 30 GPU hours to train a reasoning model respectively. The lower panel breaks down \methodterm costs by stage, showing that the actual OPD training consumes only a moderate fraction of the total budget, highlighting how the minimal infrastructure requirement of \methodterm makes the training highly efficient.}
    \label{tab:cost}
    \begin{tabular}{l|c|c}
    \toprule
    \textbf{Method} & \textit{Qwen3-4B-Base} & \textit{Qwen3-8B-Base} \\
    \midrule
    OPD                  & 72 & 120 \\
    Lightning OPD & 20 & 30 \\
    \midrule
    Speedup              & \textbf{\speedupsmall} & \textbf{\speedup} \\
    \midrule
    \multicolumn{3}{l}{\emph{Lightning OPD breakdown}} \\
    \midrule
    $\llcorner$ Rollout collection           & 10 & 10 \\
    $\llcorner$ Teacher logprob precompute   & 2 & 4 \\
    $\llcorner$ OPD training                 & 8 & 16 \\
    \bottomrule
    \end{tabular}
    \end{table}

\subsection{Scaling to Mixture-of-Experts}
\label{subsec:moe}

We further apply \methodterm to Qwen3-30B-A3B-Base, a 30B-parameter MoE model with 3B active parameters, using Qwen3-30B-A3B-Thinking-2507 as the teacher. We follow the same two-stage pipeline and training settings as the 8B experiments in Section~\ref{subsec:setup}. Standard OPD is infeasible at this scale on a single 8$\times$H100 node, as co-hosting both a 30B student and a 30B teacher for training and scoring exceeds available GPU memory. \methodterm eliminates this bottleneck by precomputing teacher log-probabilities offline, allowing all GPUs to be dedicated to student training. As shown in Table~\ref{tab:moe}, \methodterm reaches 71.0\% on AIME 2024 and 60.8\% on LiveCodeBench v5, achieving state-of-the-art performance among open MoE models at such model scale.

\begin{table*}[t]
    \centering
    \caption{\methodterm applied to a Mixture-of-Experts architecture (Qwen3-30B-A3B). By precomputing teacher log-probabilities offline, \methodterm enables on-policy distillation of a 30B MoE model on a single 8$\times$H100 node, achieving state-of-the-art performance on math and code reasoning tasks. In contrast, OPD runs out of memory as it requires co-hosting both a 30B teacher and a 30B student. \textbf{Bold} indicates improvement over the SFT baseline.}
    \label{tab:moe}
    \setlength{\tabcolsep}{3pt}
    \begin{tabular}{lcccccccc}
    \toprule
    \multirow{2}{*}{\textbf{Method}}
    & \multicolumn{4}{c}{\textbf{Math Reasoning}} & \multicolumn{3}{c}{\textbf{Code Generation}} \\
    \cmidrule(lr){2-5}\cmidrule(lr){6-8}
    & \textbf{AIME 2024} & \textbf{AIME 2025} & \textbf{HMMT 2025} & \textbf{Avg.}
    & \textbf{LCB v5} & \textbf{LCB v6} & \textbf{Avg.} \\
    \midrule
    \multicolumn{8}{l}{\emph{Student: Qwen3-30B-A3B-Base\quad Teacher: Qwen3-30B-A3B-Thinking-2507}} \\
    \midrule
    SFT                   & 66.8 & 63.2 & 44.6 & 58.2 & 39.4 & 33.0 & 36.2 \\
    OPD (OOM)       & \ding{55} & \ding{55} & \ding{55} & \ding{55} & \ding{55} & \ding{55} & \ding{55}  \\
    \rowcolor{mitblue}
    \methodterm     & \textbf{71.0} & \textbf{66.3} & \textbf{48.3} & \textbf{61.9} & \textbf{60.8} & \textbf{54.4} & \textbf{57.6} \\
    \bottomrule
    \end{tabular}
\end{table*}

\begin{table*}[t]
    \centering
    \caption{Teacher consistency ablation (AIME 2024 pass@1, \%). Rows indicate the SFT-stage teacher and columns indicate the OPD-stage teacher. The consistent setting (\textbf{diagonal, bold}) always achieves the best performance for both standard OPD and \methodterm. Teacher inconsistency degrades \methodterm more severely than standard OPD (up to 7 points at 8B), because the fixed rollout distribution compounds the gradient bias when the SFT and OPD teachers differ.}
    \label{tab:ablation_teacher}
    \setlength{\tabcolsep}{5pt}

    \vspace{-0.1cm}
    \begin{subtable}[t]{\linewidth}
    \centering
    \begin{tabular}{l | cc | l | cc}
    \toprule
    \textbf{\emph{Standard OPD}} & \multicolumn{2}{c|}{OPD Teacher} & & \multicolumn{2}{c}{OPD Teacher} \\
    \midrule
    SFT Teacher & Qwen3-8B & QwQ-32B & SFT Teacher & Qwen3-32B & QwQ-32B \\
    \midrule
    \multicolumn{2}{l}{\emph{Student: Qwen3-4B-Base}} & & \multicolumn{2}{l}{\emph{Student: Qwen3-8B-Base}} \\
    \midrule
    \quad Qwen3-8B & \textbf{65.4} & 62.4 & Qwen3-32B & \textbf{68.5} & 64.8 \\
    \quad QwQ-32B  & 61.2 & \textbf{62.8} & QwQ-32B & 65.0 & \textbf{66.5} \\
    \bottomrule
    \end{tabular}
    \end{subtable}

    \vspace{0.1cm}

    \begin{subtable}[t]{\linewidth}
    \centering
    \begin{tabular}{l | cc | l | cc}
    \toprule
    \textbf{\emph{Lightning OPD}} & \multicolumn{2}{c|}{OPD Teacher} & & \multicolumn{2}{c}{OPD Teacher} \\
    \midrule
    SFT Teacher & Qwen3-8B & QwQ-32B & SFT Teacher & Qwen3-32B & QwQ-32B \\
    \midrule
    \multicolumn{2}{l}{\emph{Student: Qwen3-4B-Base}} & & \multicolumn{2}{l}{\emph{Student: Qwen3-8B-Base}} \\
    \midrule
    \quad Qwen3-8B & \textbf{68.1} & 62.5 & Qwen3-32B & \textbf{69.9} & 63.1 \\
    \quad QwQ-32B  & 59.3 & \textbf{63.1} & QwQ-32B & 62.1 & \textbf{68.7} \\
    \bottomrule
    \end{tabular}
    \end{subtable}

    \vspace{-0.3cm}
\end{table*}

\subsection{Ablation Study}

\paragraph{Teacher consistency.}
\looseness=-1
To empirically verify teacher consistency, we cross SFT-stage and OPD-stage teacher choices by introducing QwQ-32B~\cite{qwq32b} alongside Qwen3-32B at the 8B scale and Qwen3-8B at the 4B scale, yielding a full grid of teacher combinations for both standard OPD and \methodterm in Tables~\ref{tab:ablation_teacher}. The results uniformly confirm our theoretical prediction that the teacher-consistent setting on the diagonal always achieves the best performance, while mismatched teachers consistently degrade accuracy. Notably, teacher inconsistency imposes a larger penalty on \methodterm than on standard OPD. At the 8B scale, mismatching from Qwen3-32B SFT to QwQ-32B OPD drops \methodterm by 6.8 points but standard OPD by only 3.7 points. This asymmetry arises because a mismatched SFT teacher corrupts $\pi_{\text{ref}}$ in two roles simultaneously for \methodterm, both as the reference distribution and as the source of fixed rollouts, whereas standard OPD refreshes rollouts from the current student at every step and can partially recover. Teacher consistency is consequently a more critical design requirement for \methodterm than for standard OPD.

\section{Conclusion}
\label{sec:conclusion}

We presented \methodterm, an offline on-policy distillation framework for efficient LLM post-training. By precomputing teacher log-probabilities once over SFT rollouts, \methodterm reduces the infrastructure of OPD to a single standard training job while preserving the dense per-token supervision that makes OPD effective. We further identified teacher consistency as a necessary condition for OPD to work well in general, proving that mismatched SFT and OPD teachers introduce an irreducible gradient bias that causes both paradigms to converge to a suboptimal fixed point, and that under consistency \methodterm provably shares the same optimum as standard OPD. Empirically, \methodterm trained an 8B reasoning model to 69.9\% on AIME 2024 in just 30 GPU hours, achieving \speedup speedup over standard OPD while matching its performance across all benchmarks. We hope these results encourage broader adoption of on-policy distillation in settings where a persistent teacher server is impractical, and that the teacher consistency principle provides a useful design guideline for future OPD research. More generally, our results suggest that much of the practical benefit of on-policy distillation can be retained without the conventional deployment burden of maintaining a live teacher throughout training.

{
    \small
    \bibliographystyle{unsrt}
    \bibliography{arxiv}
}

\clearpage
\appendix

\appendix

\section{Proofs}
\label{app:proofs}

We provide complete proofs of all theoretical results stated in Section~\ref{subsec:theory}. We use Assumptions~\ref{asm:second_moment}--\ref{asm:score} (and Assumption~\ref{asm:mismatch} for the teacher-mismatch results) as stated in Section~\ref{subsec:theory}. We begin with a fundamental preliminary result and then prove each theorem in order.

\paragraph{Notation.} Let $\Omega$ be a finite vocabulary. A prompt $q \in \Omega^*$ is drawn from $p(q)$. A response $x = (a_1, \ldots, a_T)$ is generated as $\pi_\theta(x \mid q) = \prod_{t=1}^{T} \pi_\theta(a_t \mid s_t)$, where $s_t = (q, a_1, \ldots, a_{t-1})$. The sequence-level IS ratio is $w(x;\theta) := \pi_\theta(x \mid q) / \pi_{\text{ref}}(x \mid q)$.

\subsection{Preliminary: Importance-Sampling Decomposition}
\label{app:is_decomp}

We begin with a fundamental identity relating the online and offline gradients via IS reweighting. This is used in all subsequent proofs.

\begin{proposition}[IS Decomposition]
\label{prop:is}
Under Assumption~\ref{asm:support}, the online gradient equals the IS-reweighted offline expectation:
\[
    \nabla J_{\text{on}}(\theta)
    = \mathbb{E}_{x \sim \pi_{\text{ref}}} \!\left[
        w(x;\theta) \cdot \sum_{t=1}^{T} A_t(\theta) \cdot \nabla \log \pi_\theta(a_t \mid s_t)
    \right].
\]
The offline gradient $\nabla J_{\text{off}}(\theta) = \mathbb{E}_{x \sim \pi_{\text{ref}}}[\sum_t A_t(\theta) \cdot \nabla \log \pi_\theta(a_t \mid s_t)]$ is the special case $w \equiv 1$.
\end{proposition}

\begin{proof}
Since $\Omega$ is a finite vocabulary, all expectations are finite sums. For any function $f(x)$:
\[
    \mathbb{E}_{x \sim \pi_\theta}[f(x)]
    = \sum_{x} \pi_\theta(x)\, f(x)
    = \sum_{x} \pi_{\text{ref}}(x) \cdot \frac{\pi_\theta(x)}{\pi_{\text{ref}}(x)} \cdot f(x)
    = \mathbb{E}_{x \sim \pi_{\text{ref}}}[w(x;\theta)\, f(x)],
\]
where the second equality requires $\pi_{\text{ref}}(x) > 0$ wherever $\pi_\theta(x) > 0$, which holds by Assumption~\ref{asm:support}. Applying this with $f(x) = \sum_t A_t(\theta) \cdot \nabla \log \pi_\theta(a_t \mid s_t)$ (stop-gradient applied to $A_t(\theta)$) gives the result.
\end{proof}

\paragraph{Remark on the surrogate gradient.}
The per-trajectory gradient $f(x;\theta) = \sum_t A_t(\theta) \cdot \nabla \log \pi_\theta(a_t \mid s_t)$ treats the advantage $A_t(\theta)$ as a fixed scalar, i.e., stop-gradient is applied to $A_t$. This is the standard update rule used in all OPD implementations~\cite{agarwal2024policy,lu2025onpolicydistillation,yang2025qwen3} and yields an advantage-weighted policy gradient. We denote these surrogate gradients by $\nabla J_{\text{on}}$ and $\nabla J_{\text{off}}$ throughout for notational consistency with the OPD literature.

\subsection{Proof of Theorem~\ref{thm:gap}}
\label{app:proof_gap}

\begin{proof}
Define the per-trajectory gradient $f(x;\theta) := \sum_{t=1}^{T} A_t(\theta) \cdot \nabla \log \pi_\theta(a_t \mid s_t)$. By Proposition~\ref{prop:is}:
\[
    \nabla J_{\text{on}}(\theta) - \nabla J_{\text{off}}(\theta)
    = \mathbb{E}_{x \sim \pi_{\text{ref}}}\!\left[(w(x;\theta) - 1)\, f(x;\theta)\right].
\]
Applying the Cauchy--Schwarz inequality to the expectation:
\[
    \left\|\mathbb{E}_{\pi_{\text{ref}}}[(w - 1)\,f]\right\|_2
    \leq \sqrt{\mathbb{E}_{\pi_{\text{ref}}}\!\left[(w-1)^2\right]}
    \cdot \sqrt{\mathbb{E}_{\pi_{\text{ref}}}\!\left[\|f\|_2^2\right]}.
\]

\paragraph{First factor.}
Since $\mathbb{E}_{\pi_{\text{ref}}}[w(x;\theta)] = \sum_x \pi_\theta(x) = 1$:
\[
    \mathbb{E}_{\pi_{\text{ref}}}\!\left[(w - 1)^2\right]
    = \mathbb{E}_{\pi_{\text{ref}}}\!\left[w^2\right] - 2\,\mathbb{E}_{\pi_{\text{ref}}}[w] + 1
    = \mathbb{E}_{\pi_{\text{ref}}}\!\left[w^2\right] - 1
    = \chi^2(\pi_\theta \,\|\, \pi_{\text{ref}}).
\]

\paragraph{Second factor.}
By Assumption~\ref{asm:score}, $\|\nabla \log \pi_\theta(a_t \mid s_t)\|_2 \leq G$ for each $t$. Applying the triangle inequality:
\[
    \|f(x;\theta)\|_2
    = \left\|\sum_{t=1}^T A_t(\theta) \cdot \nabla \log \pi_\theta(a_t \mid s_t)\right\|_2
    \leq G \cdot \sum_{t=1}^T |A_t(\theta)|.
\]
Therefore:
\[
    \mathbb{E}_{\pi_{\text{ref}}}\!\left[\|f(x;\theta)\|_2^2\right]
    \leq G^2 \cdot \mathbb{E}_{\pi_{\text{ref}}}\!\left[\left(\sum_{t=1}^T |A_t(\theta)|\right)^{\!2}\right]
    \leq G^2\, \sigma_A^2,
\]
where the last inequality uses Assumption~\ref{asm:second_moment}.

\paragraph{Combining.}
\[
    \left\|\nabla J_{\text{on}}(\theta) - \nabla J_{\text{off}}(\theta)\right\|_2
    \leq \sqrt{\chi^2(\pi_\theta \,\|\, \pi_{\text{ref}})} \cdot G\,\sigma_A. \qquad \square
\]
\end{proof}

\begin{corollary}[Zero Gap at Initialization]
\label{cor:zero_gap}
When $\pi_\theta = \pi_{\text{ref}}$, $\chi^2(\pi_\theta \,\|\, \pi_{\text{ref}}) = 0$, so $\|\nabla J_{\text{on}}(\theta) - \nabla J_{\text{off}}(\theta)\|_2 = 0$. Under standard smoothness conditions, the gap grows at rate $O(\eta\, k)$ after $k$ gradient steps with step size $\eta$.
\end{corollary}

\begin{proof}
At $\pi_\theta = \pi_{\text{ref}}$, $w(x;\theta) = 1$ for all $x$, so $\mathbb{E}_{\pi_{\text{ref}}}[w^2] = 1$ and $\chi^2 = 0$. The growth rate follows from a first-order Taylor expansion of $\chi^2(\pi_{\theta - \eta g} \| \pi_{\text{ref}})$ around $\theta$.
\end{proof}

\paragraph{Remark on $\chi^2$ versus KL.}
\looseness=-1
A bound via Pinsker's inequality yields $\|\nabla J_{\text{on}} - \nabla J_{\text{off}}\|_2 \leq M T G \sqrt{2\,\mathrm{KL}(\pi_\theta \| \pi_{\text{ref}})}$, where $M = \sup_t|A_t(\theta)|$. This is vacuous in practice because $M$ diverges whenever $\pi_\theta(a_t|s_t) \to 0$. Theorem~\ref{thm:gap} avoids this by using the $L^2$ constant $\sigma_A \ll MT$ and replacing KL with $\chi^2$. Although $\chi^2 \geq \mathrm{KL}$, advantage clipping and KL regularization keep $\chi^2$ small throughout training.

\subsection{Proof of Theorem~\ref{thm:fixed_point}}
\label{app:proof_fixed_point}

\begin{proof}
We first establish the identity $J_{\text{on}}(\theta) = -\mathrm{KL}(\pi_\theta \,\|\, \pi_T)$. By definition:
\[
    J_{\text{on}}(\theta)
    = \mathbb{E}_{x \sim \pi_\theta}\!\left[\sum_{t=1}^T \bigl(\log \pi_T(a_t \mid s_t) - \log \pi_\theta(a_t \mid s_t)\bigr)\right]
    = \mathbb{E}_{x \sim \pi_\theta}\!\left[\log \frac{\pi_T(x \mid q)}{\pi_\theta(x \mid q)}\right]
    = -\mathrm{KL}(\pi_\theta \,\|\, \pi_T),
\]
where we used the autoregressive factorization $\pi_\theta(x \mid q) = \prod_t \pi_\theta(a_t \mid s_t)$. Since $\mathrm{KL} \geq 0$, we have $J_{\text{on}}(\theta) \leq 0$ with equality if and only if $\pi_\theta = \pi_T$ on $\operatorname{supp}(\pi_\theta)$. Therefore, any $\theta^* \in \arg\min_{\theta \in \Theta} \mathrm{KL}(\pi_\theta \,\|\, \pi_T)$ is a global maximizer of $J_{\text{on}}$.

\paragraph{Achievable case ($\pi_T \in \Pi_\Theta$).}
When $\pi_T \in \Pi_\Theta$, the global maximizer $\theta^*$ achieves $\pi_{\theta^*} = \pi_T$, so $A_t(\theta^*) = \log \pi_T(a_t \mid s_t) - \log \pi_{\theta^*}(a_t \mid s_t) = 0$ for all $t$ and all $(q, x)$. Both the online and offline OPD updates vanish:
\[
    \nabla J_{\text{on}}(\theta^*) = \mathbb{E}_{\pi_{\theta^*}}\!\left[\sum_t 0 \cdot \nabla \log \pi_{\theta^*}(a_t \mid s_t)\right] = \mathbf{0}, \qquad
    \nabla J_{\text{off}}(\theta^*) = \mathbb{E}_{\pi_{\text{ref}}}\!\left[\sum_t 0 \cdot \nabla \log \pi_{\theta^*}(a_t \mid s_t)\right] = \mathbf{0}.
\]
Thus $\theta^*$ is a shared fixed point of both standard OPD and Lightning OPD.

\paragraph{Capacity-limited case ($\pi_T \notin \Pi_\Theta$).}
When $\pi_T$ cannot be exactly represented in $\Pi_\Theta$, the shared fixed-point argument above does not directly apply. However, the irreducible approximation error $\varepsilon_{\mathrm{approx}} = \min_{\theta \in \Theta} \mathrm{KL}(\pi_\theta \,\|\, \pi_T) > 0$ lower-bounds both methods, and Theorem~\ref{thm:gap} ensures that the per-step gradient discrepancy remains controlled by the policy drift $\chi^2(\pi_\theta \,\|\, \pi_{\text{ref}})$. As confirmed empirically in Figure~\ref{fig:dyn_iw}, this drift stays small throughout training, so the offline and online updates remain close in practice.
\end{proof}

\paragraph{Heuristic error decomposition.}
Combining Theorems~\ref{thm:gap} and~\ref{thm:fixed_point}, we can informally decompose the final student's divergence from the teacher as:
\[
    \mathrm{KL}(\pi_{\theta_{\text{final}}} \,\|\, \pi_T)
    \approx
    \underbrace{\varepsilon_{\text{approx}}}_{\substack{\text{model capacity} \\ \text{(irreducible)}}}
    +
    \underbrace{\varepsilon_{\text{opt}}}_{\substack{\text{optimisation error} \\ (\to 0\text{ with training})}}
    +
    \underbrace{O\!\left(G\,\sigma_A\sqrt{\chi^2(\pi_\theta\|\pi_{\text{ref}})}\right)}_{\text{offline-online distribution gap}},
\]
where $\varepsilon_{\text{approx}} = \min_{\theta \in \Theta} \mathrm{KL}(\pi_\theta \| \pi_T)$. Switching from offline to online rollouts reduces only the third term, not $\varepsilon_{\text{approx}}$, which is the dominant term determined by model capacity. This decomposition is not a formal bound for the surrogate gradient dynamics, but provides useful intuition for why Lightning OPD performs comparably to standard OPD in practice.

\subsection{Proof of Theorem~\ref{thm:implicit_reg} (Gradient Decomposition)}
\label{app:proof_implicit_reg}

\begin{proof}
Recall $f(x;\theta) = \sum_t A_t(\theta) \cdot \nabla \log \pi_\theta(a_t \mid s_t)$. By Proposition~\ref{prop:is}:
\[
    \nabla J_{\text{on}}(\theta) = \mathbb{E}_{\pi_{\text{ref}}}[w(x;\theta) \cdot f(x;\theta)].
\]
Using the covariance decomposition $\mathbb{E}[wf] = \mathbb{E}[w]\,\mathbb{E}[f] + \operatorname{Cov}[w, f]$ and $\mathbb{E}_{\pi_{\text{ref}}}[w] = 1$:
\[
    \nabla J_{\text{on}}(\theta)
    = \mathbb{E}_{\pi_{\text{ref}}}[f] + \operatorname{Cov}_{\pi_{\text{ref}}}\!\left[w,\, f\right]
    = \nabla J_{\text{off}}(\theta) + \operatorname{Cov}_{\pi_{\text{ref}}}\!\left[w(x;\theta),\, f(x;\theta)\right].
\]
Rearranging gives $\nabla J_{\text{off}}(\theta) = \nabla J_{\text{on}}(\theta) - \operatorname{Cov}_{\pi_{\text{ref}}}[w, f]$.

\paragraph{Regularization effect of the covariance term.}
At $\pi_\theta = \pi_{\text{ref}}$, $w \equiv 1$ is constant so $\operatorname{Cov}[w, f] = \mathbf{0}$ (the zero vector) and $\nabla J_{\text{off}} = \nabla J_{\text{on}}$ exactly. As $\pi_\theta$ drifts from $\pi_{\text{ref}}$, $w$ becomes large on trajectories that $\pi_\theta$ over-weights relative to $\pi_{\text{ref}}$; on those same trajectories $f$ (a gradient vector) also tends to have large magnitude. The covariance vector $\operatorname{Cov}[w, f]$ therefore grows in norm with drift and, when projected onto the drift direction $\theta - \theta_{\text{ref}}$, opposes further movement away from $\pi_{\text{ref}}$. This produces a restoring-force effect: $\nabla J_{\text{off}}$ subtracts a component aligned with the drift direction from $\nabla J_{\text{on}}$, implicitly penalizing deviation from $\pi_{\text{ref}}$ without requiring an explicit KL regularization term.
\end{proof}

\subsection{Proof of Theorem~\ref{thm:consistency_gap}}
\label{app:proof_consistency_gap}

We consider the case where the SFT-stage teacher $\pi_T^{\text{SFT}}$ and the OPD-stage reference teacher $\pi_T^{\text{OPD}}$ differ. The SFT model $\pi_{\text{ref}}$ was trained on data generated by $\pi_T^{\text{SFT}}$, so its distribution concentrates on trajectories preferred by $\pi_T^{\text{SFT}}$. However, the advantage function $A_t(\theta) = \log \pi_T^{\text{OPD}}(a_t \mid s_t) - \log \pi_\theta(a_t \mid s_t)$ reflects $\pi_T^{\text{OPD}}$. This mismatch introduces an irreducible additive bias.

\begin{proof}
Define $\delta := \chi^2(\pi_T^{\text{SFT}} \| \pi_T^{\text{OPD}})$ and $\Delta_t := \log \pi_T^{\text{SFT}}(a_t \mid s_t) - \log \pi_T^{\text{OPD}}(a_t \mid s_t)$. The advantage under teacher inconsistency decomposes as:
\[
    A_t(\theta)
    = \underbrace{\left(\log \pi_T^{\text{SFT}}(a_t \mid s_t) - \log \pi_\theta(a_t \mid s_t)\right)}_{\text{consistent advantage } A_t^{\text{cons}}(\theta)}
      - \Delta_t.
\]
Correspondingly, define $f_{\text{cons}}(x;\theta) = \sum_t A_t^{\text{cons}}(\theta) \cdot \nabla \log \pi_\theta(a_t \mid s_t)$ and $f_\Delta(x;\theta) = -\sum_t \Delta_t \cdot \nabla \log \pi_\theta(a_t \mid s_t)$, so $f = f_{\text{cons}} + f_\Delta$.

By Proposition~\ref{prop:is} and the triangle inequality:
\[
    \left\|\nabla J_{\text{on}}(\theta) - \nabla J_{\text{off}}(\theta)\right\|_2
    = \left\|\mathbb{E}_{\pi_{\text{ref}}}[(w-1)f]\right\|_2
    \leq \underbrace{\left\|\mathbb{E}_{\pi_{\text{ref}}}[(w-1)f_{\text{cons}}]\right\|_2}_{\text{(I)}}
      + \underbrace{\left\|\mathbb{E}_{\pi_{\text{ref}}}[(w-1)f_\Delta]\right\|_2}_{\text{(II)}}.
\]

\paragraph{Term (I): Consistent part.}
This is identical to the proof of Theorem~\ref{thm:gap} with the consistent advantage, giving:
\[
    \text{(I)} \leq G\,\sigma_A \cdot \sqrt{\chi^2(\pi_\theta \,\|\, \pi_{\text{ref}})}.
\]

\paragraph{Term (II): Mismatch part.}
By the triangle inequality and Assumption~\ref{asm:score}, $\|f_\Delta(x;\theta)\|_2 \leq G \cdot \sum_t |\Delta_t|$. Applying Cauchy--Schwarz:
\[
    \text{(II)}
    \leq \sqrt{\chi^2(\pi_\theta\|\pi_{\text{ref}})} \cdot G \cdot \sqrt{\mathbb{E}_{\pi_{\text{ref}}}\!\left[\left(\sum_t |\Delta_t|\right)^{\!2}\right]}
    \leq G\,\sigma_\Delta \cdot \sqrt{\chi^2(\pi_\theta\|\pi_{\text{ref}})},
\]
where the last step applies Assumption~\ref{asm:mismatch}. Combining Terms (I) and (II):
\[
    \left\|\nabla J_{\text{on}}(\theta) - \nabla J_{\text{off}}(\theta)\right\|_2
    \leq G \cdot (\sigma_A + \sigma_\Delta) \cdot \sqrt{\chi^2(\pi_\theta \,\|\, \pi_{\text{ref}})}.
\]
At $\pi_\theta = \pi_{\text{ref}}$ (where $w \equiv 1$), $\chi^2 = 0$ and the online-offline gap is exactly zero, regardless of teacher mismatch.

\paragraph{Offline gradient bias.}
Separately, the mismatched offline gradient carries a persistent bias relative to the consistent offline gradient. Writing $\nabla J_{\text{off}} = \mathbb{E}_{\pi_{\text{ref}}}[f_{\text{cons}} + f_\Delta]$ and $\nabla J_{\text{off}}^{\delta=0} = \mathbb{E}_{\pi_{\text{ref}}}[f_{\text{cons}}]$, their difference is $\mathbb{E}_{\pi_{\text{ref}}}[f_\Delta]$, which can be nonzero when $\sigma_\Delta > 0$. By the triangle inequality, Jensen's inequality, and Assumption~\ref{asm:mismatch}:
\[
    \left\|\nabla J_{\text{off}} - \nabla J_{\text{off}}^{\delta=0}\right\|_2
    = \left\|\mathbb{E}_{\pi_{\text{ref}}}[f_\Delta]\right\|_2
    \leq \mathbb{E}_{\pi_{\text{ref}}}\!\left[\|f_\Delta\|_2\right]
    \leq G \cdot \mathbb{E}_{\pi_{\text{ref}}}\!\left[\sum_t |\Delta_t|\right]
    \leq G\,\sigma_\Delta.
\]
This bias is independent of $\chi^2$ and persists even at initialization, corrupting the offline update direction throughout training.
\end{proof}

\subsection{Proof of Theorem~\ref{thm:consistency_online}}
\label{app:proof_consistency_online}

\begin{proof}
Let $\nabla J_{\text{on}}^{\delta=0}(\theta)$ denote the standard OPD gradient under a consistent teacher ($\pi_T^{\text{SFT}} = \pi_T^{\text{OPD}}$), and let $\nabla J_{\text{on}}(\theta)$ denote the gradient under the mismatched teacher $\pi_T^{\text{OPD}} \neq \pi_T^{\text{SFT}}$. For any iterate $\theta$ reachable from $\pi_{\text{ref}}$, their difference is:
\begin{align*}
    \nabla J_{\text{on}}(\theta) - \nabla J_{\text{on}}^{\delta=0}(\theta)
    &= \mathbb{E}_{x \sim \pi_\theta}\!\left[\sum_t \left(A_t^{\text{OPD}}(\theta) - A_t^{\text{SFT}}(\theta)\right) \cdot \nabla \log \pi_\theta(a_t \mid s_t)\right] \\
    &= -\mathbb{E}_{x \sim \pi_\theta}\!\left[\sum_t \Delta_t \cdot \nabla \log \pi_\theta(a_t \mid s_t)\right],
\end{align*}
where $\Delta_t = \log \pi_T^{\text{SFT}}(a_t \mid s_t) - \log \pi_T^{\text{OPD}}(a_t \mid s_t)$.

By the triangle inequality and Assumption~\ref{asm:score}, $\|\sum_t \Delta_t \cdot \nabla \log \pi_\theta(a_t \mid s_t)\|_2 \leq G \cdot \sum_t |\Delta_t|$. Taking the norm outside the expectation and applying Jensen's inequality:
\[
    \left\|\nabla J_{\text{on}}(\theta) - \nabla J_{\text{on}}^{\delta=0}(\theta)\right\|_2
    \leq G \cdot \mathbb{E}_{x \sim \pi_\theta}\!\left[\sum_t |\Delta_t|\right]
    \leq G \cdot \sqrt{\mathbb{E}_{x \sim \pi_\theta}\!\left[\left(\sum_t |\Delta_t|\right)^{\!2}\right]}.
\]
Since $\theta$ is reachable from $\pi_{\text{ref}}$, by Assumption~\ref{asm:support} $\pi_\theta \ll \pi_{\text{ref}}$. Applying the IS identity (Proposition~\ref{prop:is}):
\[
    \mathbb{E}_{x \sim \pi_\theta}\!\left[\left(\sum_t |\Delta_t|\right)^{\!2}\right]
    = \mathbb{E}_{x \sim \pi_{\text{ref}}}\!\left[w(x;\theta)\left(\sum_t |\Delta_t|\right)^{\!2}\right].
\]
At initialization $\theta = \theta_{\text{ref}}$, we have $w \equiv 1$, so this reduces exactly to $\mathbb{E}_{\pi_{\text{ref}}}[(\sum_t |\Delta_t|)^2] \leq \sigma_\Delta^2$ by Assumption~\ref{asm:mismatch}. Substituting:
\[
    \left\|\nabla J_{\text{on}}(\theta_{\text{ref}}) - \nabla J_{\text{on}}^{\delta=0}(\theta_{\text{ref}})\right\|_2
    \leq G\,\sigma_\Delta. \qquad \square
\]
\end{proof}

\paragraph{Remark (extension beyond initialization).}
For iterates $\theta$ with bounded drift ($\chi^2(\pi_\theta \,\|\, \pi_{\text{ref}}) \leq C$), we expect the bound to remain of similar order. Intuitively, $\Delta_t$ depends only on the two teachers and the token sequence, not on $\theta$, so the covariance between $w$ and $(\sum|\Delta_t|)^2$ stays small when drift is small. The term $G\,\sigma_\Delta$ thus biases the gradient of standard OPD whenever $\sigma_\Delta > 0$, degrading the effective convergence point relative to the fixed point of Theorem~\ref{thm:fixed_point}.

\section{Implementation Details}
\label{app:hyperparams}

Tables~\ref{tab:sft_hyperparams} and~\ref{tab:opd_hyperparams} list the full hyperparameter configurations for the SFT and OPD stages respectively. For the OPD stage, the same hyperparameters are used for both the math and code domains. We set the maximum response length to 4,096 tokens during OPD training. Although the evaluation generation length is significantly longer (up to 40,960 tokens for code benchmarks), we find that training with 4,096 tokens already achieves optimal performance while offering substantially better training efficiency; increasing the rollout length beyond this threshold does not improve results. For code generation, the OPD stage is initialized from the math-trained OPD checkpoint rather than the SFT checkpoint directly. We find this consistently outperforms initializing code OPD from the SFT model, consistent with prior findings that math reasoning training provides a stronger initialization for code training~\cite{Nemotron_Cascade}.

\begin{table}[ht]
    \centering
    \caption{SFT stage hyperparameters.}
    \label{tab:sft_hyperparams}
    \setlength{\tabcolsep}{8pt}
    \begin{tabular}{lcc}
    \toprule
    \textbf{Hyperparameter} & \textbf{4B Scale} & \textbf{8B Scale} \\
    \midrule
    Training steps          & 3000 & 3000 \\
    Global batch size       & 256 & 128 \\
    Max sequence length     & 16384 & 16384 \\
    Learning rate           & $8 \times 10^{-5}$ & $8 \times 10^{-5}$ \\
    LR schedule             & cosine & cosine \\
    Warmup ratio            & 0.1 & 0.1 \\
    Packing                 & \checkmark & \checkmark \\
    DeepSpeed stage         & ZeRO-0 & ZeRO-1 \\
    \bottomrule
    \end{tabular}
\end{table}

\begin{table}[h]
    \centering
    \caption{OPD stage hyperparameters for standard OPD and \methodterm. The same configuration is used for both the math and code domains.}
    \label{tab:opd_hyperparams}
    \setlength{\tabcolsep}{10pt}
    \begin{tabular}{lcc}
    \toprule
    \textbf{Hyperparameter} & \textbf{4B Scale} & \textbf{8B Scale} \\
    \midrule
    Training steps          & 150 & 150 \\
    Global batch size       & 256 & 256 \\
    Max response length     & 4096 & 4096 \\
    Learning rate           & $2 \times 10^{-6}$ & $2 \times 10^{-6}$ \\
    LR schedule             & constant & constant \\
    Weight decay            & 0.1 & 0.1 \\
    Adam $\beta_1$          & 0.9 & 0.9 \\
    Adam $\beta_2$          & 0.98 & 0.98 \\
    Rollout temperature     & 0.8 & 0.8 \\
    Rollout top-$p$         & 1.0 & 1.0 \\
    Advantage clip range    & $[-10, 10]$ & $[-10, 10]$ \\
    Tensor parallel size    & 2 & 4 \\
    \bottomrule
    \end{tabular}
\end{table}

\section{Additional Experimental Analysis}
\label{app:extra_exp}

\subsection{Training Dynamics}
\label{app:training_dynamics}

\begin{figure*}[t]
    \centering
    \begin{subfigure}[t]{0.32\linewidth}
        \centering
        \includegraphics[width=\linewidth]{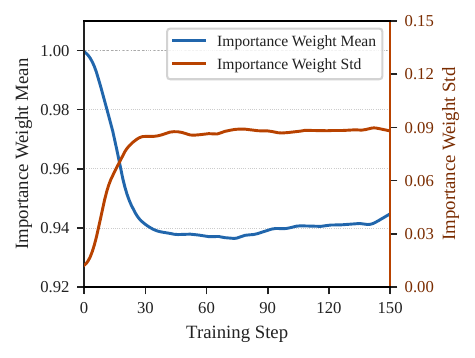}
        \caption{Importance weight dynamics.}
        \label{fig:dyn_iw}
    \end{subfigure}
    \hfill
    \begin{subfigure}[t]{0.32\linewidth}
        \centering
        \includegraphics[width=\linewidth]{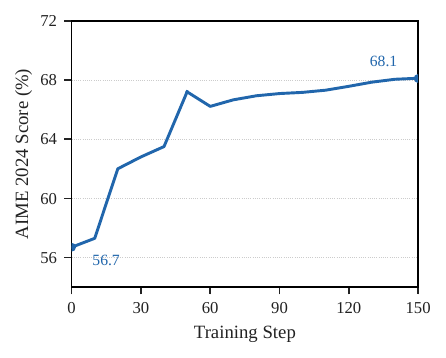}
        \caption{AIME 2024 score vs.\ training step.}
        \label{fig:dyn_aime}
    \end{subfigure}
    \hfill
    \begin{subfigure}[t]{0.32\linewidth}
        \centering
        \includegraphics[width=\linewidth]{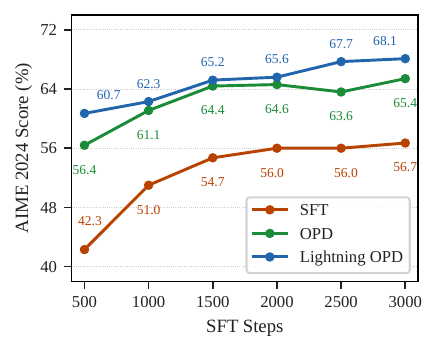}
        \caption{AIME 2024 score vs.\ SFT steps.}
        \label{fig:dyn_sft}
    \end{subfigure}
    \caption{Training dynamics of \methodterm (Qwen3-4B-Base student).
    \textbf{(a)} The mean importance weight $w_t = \pi_\theta / \pi_{\text{ref}}$ drops rapidly to $\approx 0.94$ within the first 20 steps and plateaus, with its standard deviation stabilizing concurrently, validating the implicit regularization in Theorem~\ref{thm:implicit_reg}.
    \textbf{(b)} AIME 2024 pass@1 rises steeply in the first 50 OPD steps and saturates thereafter, confirming that 150 steps is a sufficient training budget.
    \textbf{(c)} AIME 2024 pass@1 of the SFT model, standard OPD, and \methodterm as a function of SFT checkpoint quality. All three improve consistently with more SFT steps, and both OPD variants provide a large, stable gain on top of the SFT baseline at every checkpoint.}
    \label{fig:training_dynamics}
\end{figure*}

\looseness=-1
Figure~\ref{fig:training_dynamics} examines the internal dynamics of \methodterm across both stages of training. Figure~\ref{fig:dyn_iw} tracks the per-token importance weight $w_t = \pi_\theta / \pi_{\text{ref}}$ throughout OPD training. The mean drops from 1 to 0.94 within the first 20 steps and then plateaus, while the standard deviation rises sharply in the same early phase before stabilizing at a moderate level. The mean remaining close to 0.94 indicates that the student policy stays near the reference distribution throughout training, and the standard deviation stabilizing below 0.1 indicates that per-token distributional shift is consistently small. Both quantities remaining bounded confirms the implicit regularization of Theorem~\ref{thm:implicit_reg}, which shows that fixing rollouts to $\pi_{\text{ref}}$ automatically constrains both the magnitude and spread of policy drift without any explicit KL penalty. Figure~\ref{fig:dyn_aime} shows the AIME 2024 score throughout OPD training. The score converges remarkably fast, with the student capturing nearly all of its performance gain within the first 50 steps and remaining stable thereafter, justifying our choice of 150 steps as a sufficient training budget. Figure~\ref{fig:dyn_sft} shows the effect of SFT checkpoint quality on the final model. All three curves improve consistently with more SFT steps, and both OPD variants provide a large, stable gain on top of the SFT baseline at every checkpoint. The relative ordering of the three methods remains consistent across all SFT budgets, indicating that \methodterm is robust to the choice of SFT training length.

\section{Extended Discussion}
\label{app:discussion}

We discuss how Lightning OPD relates to two superficially similar paradigms, offline RL and offline knowledge distillation, and clarify why neither subsumes our approach.

\paragraph{Relation to Offline RL.}
\label{remark:offline_rl}
Lightning OPD resembles offline RL in that both optimize over a fixed dataset, but the resemblance is superficial and a direct application of offline RL techniques to OPD would fail for reasons that offline RL methods are not designed to address.
The central challenge of offline RL is OOD action overestimation arising from sparse reward signals, which offline RL methods address through conservatism mechanisms such as value pessimism or policy constraint.
Neither problem exists in Lightning OPD, where the teacher supplies dense per-token log-probability supervision for all token sequences, leaving no OOD region and no sparse reward to cause high-variance estimation.
Conservatism is therefore neither necessary nor applicable.
The real obstacle to offline OPD is teacher inconsistency. When $\sigma_\Delta > 0$, the irreducible bias $G\sigma_\Delta$ is a structural property of the gradient field itself, not an estimation artifact from limited data coverage, and no importance sampling correction or conservatism mechanism can remove it.
Beyond the challenge, the two paradigms also differ in the nature of their solutions. Offline RL fixed points are data-limited, shaped by the coverage of the behavior policy, so better data coverage leads to better policies. Theorem~\ref{thm:fixed_point} and our empirical results suggest that Lightning OPD's performance ceiling is primarily capacity-limited, determined by model capacity $\Pi_\Theta$ relative to the teacher rather than by the rollout distribution. Improving the rollout distribution cannot push past the teacher's own capability, and the right lever is model capacity, not data coverage.
In summary, Lightning OPD is not a variant of offline RL applied to distillation. It is a principled offline approximation to an on-policy distillation objective, one whose unique challenge is teacher inconsistency rather than distributional coverage, and one that recovers standard OPD when teacher consistency is enforced.

\paragraph{Relation to Offline Knowledge Distillation.}
\label{remark:offline_kd}
Lightning OPD also differs fundamentally from offline (off-policy) knowledge distillation \cite{kim2016sequence,gu2024minillm}, despite the shared use of precomputed teacher signals.
Offline KD trains the student on teacher-generated sequences, so the student only receives supervision on trajectories the teacher would produce, never on its own mistakes. Lightning OPD instead collects rollouts from the student's own policy $\pi_{\text{ref}}$ and evaluates the teacher's per-token log-probabilities on these student-generated sequences. This means the teacher provides corrective signals on exactly the distribution the student will encounter during inference, which is the core advantage of on-policy methods over off-policy ones \cite{agarwal2024policy}.
The distinction is preserved in Lightning OPD even though teacher log-probabilities are precomputed. Empirically, the OPD stage provides substantial gains over the SFT baseline across all benchmarks, confirming that on-policy supervision on student rollouts extracts significantly more from the teacher than offline KD on teacher-generated data alone. This is consistent with Theorem~\ref{thm:fixed_point}: Lightning OPD's performance is primarily capacity-limited, depending on how well $\pi_\theta$ can approximate $\pi_T$, rather than being restricted to the support of teacher-generated data.

\paragraph{Relation to Rang et al.~\cite{rang2025revealing}.}
\label{remark:rang2025}
Rang et al.~\cite{rang2025revealing} also precompute teacher signals over student-generated responses and train without a live teacher server. While the high-level motivation of offline on-policy distillation is shared, the two approaches differ in a fundamental way: the coupling between SFT and distillation. Rang et al.\ treat SFT and knowledge distillation as independent sequential stages: SFT is performed on independently curated data, and the distillation teacher is a separate model with no constraint linking the two stages. No analysis is provided for when or why the offline approximation is reliable. A central finding of Lightning OPD is that the two stages must be considered holistically. We identify teacher consistency as an important design principle, requiring that the SFT data be generated by the same teacher used for OPD, and show that violating it introduces a gradient bias that degrades both offline and online OPD (Theorems~\ref{thm:consistency_gap}--\ref{thm:consistency_online}). This bias is a structural property of the gradient field itself, not an artifact of the offline approximation, and no amount of data or training can remove it. Enforcing teacher consistency is sufficient for offline OPD to share the same fixed point as online OPD when the teacher is representable (Theorems~\ref{thm:gap},~\ref{thm:fixed_point},~\ref{thm:implicit_reg}), providing formal guarantees that are absent in prior work. Beyond the design principle, the two approaches also differ in the distillation paradigm. Rang et al.\ formulate distillation as supervised learning with a composite cross-entropy and KL loss using soft teacher logits, whereas Lightning OPD formulates it as policy gradient optimization where the per-token advantage $\log \pi_T - \log \pi_\theta$ drives the gradient update. The policy gradient formulation enables direct integration with standard OPD and RLVR post-training infrastructure, and is what makes the theoretical analysis (gradient discrepancy bounds, shared fixed points, implicit regularization) tractable.

\paragraph{Evaluation Templates.}
We use the following prompt templates for evaluation. For mathematical reasoning benchmarks, the prompt follows the Qwen3 chat format:

\begin{verbatim}
<|im_start|>user
Question: {problem}
Please reason step by step, and put your final answer within \boxed{}.
<|im_end|>
<|im_start|>assistant
\end{verbatim}

For code generation benchmarks, the prompt includes a system message and a task description following LiveCodeBench conventions:

\begin{verbatim}
<|im_start|>system
You are a helpful and harmless assistant. You are Qwen developed by Alibaba.
You should think step-by-step.
<|im_end|>
<|im_start|>user
You will be given a question (problem specification) and will generate a correct
Python program that matches the specification and passes all tests.

Question: {question}

Read the inputs from stdin solve the problem and write the answer to stdout
(do not directly test on the sample inputs). Enclose your code within delimiters
as follows. Ensure that when the python program runs, it reads the inputs, runs
the algorithm and writes output to STDOUT.
# YOUR CODE HERE

<|im_end|>
<|im_start|>assistant
\end{verbatim}

\section{Limitations}
\label{app:limitations}

Our experiments focus on mathematical reasoning and code generation, both of which benefit from well-defined verifiable evaluation metrics. Extending Lightning OPD to broader post-training tasks such as multi-turn agent interactions, tool use, and open-ended instruction following remains an open direction, as these settings introduce challenges including multi-turn distribution shift and the difficulty of precomputing meaningful teacher signals over long interaction trajectories. Additionally, teacher consistency requires that SFT training data be generated by the same teacher used for OPD; when adopting a new teacher, this necessitates regenerating the SFT dataset, which can be resource-intensive for large teacher models and partially offsets the training-time savings, though it remains a one-time cost amortized over multiple experiments.

\end{document}